\documentclass[12pt]{article}

\usepackage[utf8]{inputenc} % allow utf-8 input
\usepackage[T1]{fontenc}    % use 8-bit T1 fonts
\usepackage{hyperref}       % hyperlinks
\usepackage{url}            % simple URL typesetting
\usepackage{booktabs}       % professional-quality tables
\usepackage{amsfonts}       % blackboard math symbols
\usepackage{nicefrac}       % compact symbols for 1/2, etc.
\usepackage{microtype}      % microtypography
\usepackage{xcolor}         % colors

\usepackage{amsthm} % for theorem
\usepackage{amsmath}
\usepackage{algorithm} % for algorithms
\usepackage{algorithmicx} % for algorithms
\usepackage{algpseudocode} % for keywords in algorithms
\usepackage{tikz}
\usetikzlibrary{arrows.meta, positioning}

\usepackage{dsfont}

\usepackage{enumitem}
\newlist{balditemize}{itemize}{1}
\setlist[balditemize]{label=$\bullet\,\,$, wide=\parindent, labelsep*=0pt, leftmargin=*, topsep=0pt, itemsep=0pt}

\newtheorem{theorem}{Theorem}

\newcommand{\ensclass}[1]{\psi_{#1}}

\title{A statistical approach to bias in zero-shot learning: the lens of handwriting recognition
}

\date{}

\newcommand{\monek}[1]{M_{[1]}^{(#1)}}
\newcommand{\mtwoki}[2]{M_{[2],#2}^{(#1)}}
\newcommand{\dtrain}{\mathcal{D}_\mathrm{train}}
\newcommand{\dtest}{\mathcal{D}_\mathrm{test}}
\def\P{\mathbb{P}}
\def\ind{\mathds{1}}

\usepackage[bottom]{footmisc}

\begin{document}

\maketitle
\footnotetext[1]{Lead author}
{
\author{%
\begin{tabular}{cc}
\begin{minipage}[t]{0.45\textwidth}
\centering
Clarence Chew\footnotemark[1] \\[2pt]
Department of Mathematics \\
National University of Singapore \\
\texttt{clarence\_chew\_@nus.edu.sg} \\[1em]

Sukalpa Chanda \\ [2pt]
Dept of Comp Sc \& Comm \\
Østfold University College \\
\texttt{sukalpa@ieee.org}  \\[1em]

Soumendu Sundar Mukherjee \\[2pt]
Statistics and Mathematics Unit \\
Indian Statistical Institute \\
\texttt{ ssmukherjee@isical.ac.in}

\end{minipage}
&
\begin{minipage}[t]{0.45\textwidth}
\centering
Gim Siang Chia\footnotemark[1] \\[2pt]
Department of Mathematics \\
National University of Singapore \\
\texttt{e0310077@u.nus.edu} \\[1em]

Subhroshekhar Ghosh \\ [2pt]
Department of Mathematics \\
National University of Singapore \\
\texttt{subhrowork@gmail.com}

\end{minipage}
\end{tabular}
}
}

\newpage

\begin{abstract}
Generalized zero-shot learning (GZSL) has emerged as an important paradigm for visual recognition systems that must generalize to classes that were not observed during training. Traditional GZSL techniques are limited by their applicability to a relatively small number of such unseen classes, scalability beyond which is challenging due to its well-known {\it misclassification bias} towards classes observed during training. In this work, we investigate the GZSL paradigm through the lens of zero-shot handwritten word recognition over extremely large vocabularies, which brings into sharp focus the central problem of inherent bias of GZSL methods towards the training data. We propose a statistical approach to rectifying this bias, which views any classical GZSL feature learner as a black box mechanism whose intrinsic bias is reflected in its ability to identify the training status (seen vs. unseen) of a typical data point, which may be viewed as an out of distribution inferential problem. Our method leverages a simple two-stage hierarchical architecture, combining a classical GZSL blackbox in the first stage and an ensemble of light-weight Monte Carlo based bias-correctors in the second, wherein the dual tasks of bias estimation and threshold calibration are performed. Once debiased, the classification of test data is undertaken only restricted to its predicted training status, using well-founded approaches with solid statistical underpinnings (such as nearest neighbour, logistic regression and random forests). In doing so, we can achieve relative accuracy improvements of over 20\% in the classification of unseen words compared to established techniques. A key outcome of our investigations is that word recognition over large scale vocabularies can be well-represented by a much lower dimensional approximation (with as low as 15 dimensions), which renders practicable classical statistical methodologies as well as very large scale Monte Carlo techniques. Our approach is underpinned by mathematical analysis that captures the essence of the statistical approach to bias correction. It may be recognized that our statistical approach to bias rectification can be combined in a turn-key fashion with potentially any classical GZSL learner as a blackbox, thereby suggesting a wide scope of applicability of this method for a wide variety of GZSL implementations in different domains.  
\end{abstract}
% Abstract should be one paragraph

\tableofcontents

\newpage

\section{Introduction}

{\it Generalized Zero-shot Learning (ZSL)} has emerged as an important paradigm for visual recognition systems that must generalize to classes that were not observed during training \cite{xian2016generalized}. In existing approaches, semantic relationships between classes are often captured through attribute annotations or language embeddings, with the ultimate objective of transferring knowledge from seen classes to unseen ones. Most existing work on zero-shot recognition has focused on object classification benchmarks such as Animals with Attributes (AwA), CUB-200, and SUN, where the number of unseen classes typically ranges from tens to a few hundred \cite{chao2016empirical,xian2019zero,verma2018generalized,verma2020meta}.
%In these settings, semantic relationships between classes are often captured through attribute annotations or language embeddings, with the go models to transfer knowledge from seen classes to unseen ones. 

%In this work we study a substantially more challenging setting: 
A well-known issue with existing GZSL frameworks is a persistent {\it bias towards seen classes}, i.e. those classes whose data have been used for training the learning model \cite{chao2016empirical}. 
This problem of inherent bias towards seen classes is brought into sharp focus in the substantially more challenging setting of 
\emph{zero-shot handwritten word recognition over extremely large vocabularies}. 
Unlike object recognition datasets, where visual classes correspond to semantically meaningful object categories, handwritten word recognition operates over a combinatorial space of word forms. 
Even moderately sized vocabularies can contain thousands of distinct classes, and new word classes can be constructed through arbitrary combinations of characters.

This fundamental difference introduces several challenges that are not present in conventional zero-shot object recognition. 
First, the {\it number of potential classes is significantly larger}, leading to a much denser and more complex class manifold. 
Second, visual {\it differences between classes can be extremely subtle}; words may differ by only a single character or by a small modification in stroke structure. 
Third, handwritten text exhibits {\it substantial intra-class variation} due to differences in writing style, stroke connectivity, and character spacing. 
Together, these factors make large-vocabulary zero-shot word recognition a significantly more difficult problem than standard object-level zero-shot classification \cite{rai2021phoscnet,bhatt2022phoscctc}.

Most modern handwriting recognition systems rely on Connectionist Temporal Classification (CTC) based sequence models that predict character sequences \cite{graves2006connectionist}. 
While effective for transcription tasks, these models are known to exhibit several limitations which render them a structurally misaligned for zero-shot word recognition problems; these include segmentation difficulties, accumulation of errors across characters, and focus on local (as opposed to global) structures in complex words. 

To address these difficulties, we leverage a {\it structured word embedding paradigm}, namely the so-called {\it Pho(SC)Net} \cite{rai2021phoscnet}, the proposed  new representation called Pyramidal Histogram of Shapes, captures both the positional information of characters and their visual shape appearance, making it more discriminative than traditional PHOC embeddings \cite{rai2021phoscnet}. Handwritten word images are mapped directly into this embedding space, and recognition is performed by matching the predicted embedding against embeddings of candidate. Combined with a classical nearest neighbor classifier, this approach is known to prominently capture the well-known phenomenon of bias towards seen classes in GZSL, with accuracy rates for correct identification of unseen classes at a lowly 77\% (compared to a seen class accuracy of around 93\%) \cite{rai2021phoscnet}. 

\vskip5pt
\textbf{Our Contributions.}
In this paper, we investigate the problem of {\it mitigating the inherent bias towards seen classes in GZSL} through the prism of  handwritten word recognition over large vocabularies.  

We propose a {\it statistical approach to rectifying  bias} towards seen classes, which views any classical GZSL feature learner as a black box mechanism whose intrinsic bias is reflected in its ability (or lack thereof) to identify the training status (seen vs. unseen) of the class of a typical test data point. This latter issue may in fact be viewed as an out of distribution inferential problem; although in view of complete class identification (as opposed to mere training status), our problem is much more challenging and goes far beyond classical out of distribution learning. 

Our method leverages a simple {\it two-stage hierarchical architecture}, combining a classical GZSL feature learning blackbox in the first stage and an ensemble of light
weight Monte Carlo based bias-correctors in the second. While the first stage may be envisaged to generate a structured embedding of a give data point, it is the second stage where the dual tasks of bias estimation and threshold calibration are performed. To wit, using a  Monte Carlo aggregation approach, the classifiers in the second phase certify an input data point to from a seen class only if the proportion of seen certificates from the individual second stage classifiers crosses a suitable threshold that is calibrated to capture the inherent bias of the GZSL feature generator in the first stage.   The overall output of this two stage architecture is thus two-fold: (i) the structured word embedding (from the first stage) and (ii) a debiased training status predictor (seen/unseen) of the class of the input data point (from the second stage).  

Subsequent to processing via this two-stage architecture, the
classification of test data is undertaken only restricted to its predicted training
status, using well-founded approaches with solid statistical underpinnings (such 
as {\it nearest neighbour classifiers, logistic regression} and {\it random forests}). This design allows the system to dynamically restrict recognition to either seen or unseen candidate vocabularies, mitigating the well-known bias toward seen classes in generalized zero-shot learning. Our approach is well-motivated by an underpinning mathematical analysis that unveils the potential of such a method under a simplified analytical setup.

% In this work we explore an alternative formulation based on \textbf{structured word embeddings}. 
% Instead of predicting character sequences, we represent each word by a high-dimensional feature vector that encodes multiple aspects of its compositional structure, including letter shape descriptors, segment-level character presence, and bigram statistics across different parts of the word. 
% Handwritten word images are mapped directly into this embedding space, and recognition is performed by matching the predicted embedding against embeddings of candidate words.

A key observation augmenting our approach is a highly revealing geometric property of this embedding space. 
Although the structured word representation has dimensionality $d = 769$, we find that the embeddings produced by trained models lie on a remarkably low-dimensional manifold.  Empirically, applying principal component analysis (PCA) reveals that the essential structure of the embedding space can be captured using only \textit{15 dimensions}. 
Despite this dramatic reduction in dimensionality, the resulting representations remain highly informative for distinguishing between thousands of word classes. This phenomenon suggests that the embedding space possesses a strong underlying geometric structure that can be exploited for recognition. 
In particular, we can leverage this property to construct highly lightweight classifiers that operate in the reduced-dimensional space to determine whether a predicted embedding corresponds to a word class observed during training.

% Building on this insight, we introduce a two-level recognition architecture consisting of embedding models ($\monek{k}$) and seen/unseen discriminators ($\mtwoki{k}{i}$). 
% Each embedding model is trained on a subset of the available vocabulary classes and produces structured word embeddings from handwritten images. 
% The $\mtwoki{k}{i}$ classifiers operate on the low-dimensional embedding manifold to estimate whether a given prediction corresponds to a class observed during training. 
% This design allows the system to dynamically restrict recognition to either seen or unseen candidate vocabularies, mitigating the well-known bias toward seen classes in generalized zero-shot learning.

By combining structured word embeddings with a principled statistical approach to bias correction, the proposed framework enables recognition across vocabularies containing thousands of word classes, including a large number of unseen words. Augmented with dimensionality reduction leading to a parsimonious representation of the embedding manifold geometry, we are able to provide significant advances (relative improvements of over 20\%) in accurately identifying unseen classes in the challenging problem of zero short handwritten word recognition, at a very manageable computational cost. The proposed method leverages the structured word embedding Pho(SC)Net as a black box, and can thus be potentially augmented in a turn-key fashion to any classical GZSL feature learning mechanism to achieve superior outcomes, especially for correct identification of unseen data categories.

 \vskip5pt
{\bf Related Work.} 
{\it Zero-shot learning} (ZSL) has been studied extensively in object recognition, where models learns a mapping from seen to unseen categories through attributes or semantic embeddings \cite{lampert2009learning,palatucci2009zero,frome2013devise,socher2013zero,akata2015evaluation}.  In generalized zero-shot learning, inference during test phase involves recognition over both seen and unseen classes \cite{chao2016empirical,xian2019zero,verma2018generalized,verma2020meta}. While these methods have led to strong progress on benchmarks such as AwA, CUB, and SUN, the underlying setting typically involves tens or hundreds of semantically separated object classes. But Handwritten word recognition is fundamentally different: the label space is combinatorial, vocabularies can contain thousands of classes, and neighboring classes may differ by only a single character or minor stroke variation. This makes direct transfer of standard ZSL assumptions and methods non-trivial.

% \paragraph{Sequence-based handwritten text recognition:}
% Sequence transcription model based approaches  are very common in modern handwritten text recognition (HTR) and is dominated by , especially CNN--RNN architectures trained with Connectionist Temporal Classification (CTC) \cite{graves2006connectionist}. More recently, attention- or Transformer-based systems \cite{bluche2017joint,li2023trocr} are also evident. These methods are highly effective when enough labeled data are available and the objective is to achieve full transcription. However, in those methods  prediction is done as a sequence of local decisions, thus them vulnerable to error in challenging circumstances like irregular spacing, touching characters, ligatures, and degraded document quality. This weakness is especially relevant in retrieval-oriented scenarios, where even a small character-level mistake can corrupt the final word match.

{\it Embedding-based word spotting and recognition.}
In many real-world applications such as searching large digital archives, 
%the primary objective is not to fully transcribe documents but to retrieve instances of specific query words or phrases. Here, the idea is 
it is customary to represent both word images and text strings in a shared embedding space, where retrieval can be performed via similarity-based matching. 
%This formulation directly supports efficient search over large vocabularies and naturally accommodates open-vocabulary scenarios. 
The PHOC representation introduced by Almazán et al.~\cite{almazan2014word} is a foundational approach in this direction, enabling segmentation-free word spotting and lexicon-based recognition.
Subsequent work showed that deep architectures can predict such embeddings effectively from word images \cite{sudholt2017evaluating,krishnan2016deep}, and PHOC-style methods have remained attractive because they support lexicon search through nearest-neighbor matching rather than explicit sequence decoding. 

{\it Zero-shot handwritten word recognition.}
Closer to our setting, recent work has explored zero-shot handwritten word recognition using structured embeddings that encode both character occurrence and shape information. Pho(SC)Net introduced a hybrid PHO(SC) representation for zero-shot word image recognition in historical documents, showing that shape-aware structured embeddings can improve generalization to unseen words \cite{rai2021phoscnet}. Pho(SC)-CTC later combined these representations with a CTC-based decoding framework to improve recognition further \cite{bhatt2022phoscctc}. These studies established structured word embeddings as a viable basis for zero-shot handwritten recognition, but they do not explicitly address the large-scale generalized setting in which test data contain both seen and unseen words and the model must manage the strong bias toward seen classes.

% \textbf{Position of the present work.}
% While prior approaches have demonstrated that structured embeddings can scale to large vocabularies, they typically operate in the original high-dimensional embedding space and do not explicitly address the geometric structure of the learned representations or the resulting bias toward seen classes in the generalized zero-shot setting.
% In contrast, we exploit a key empirical observation: although structured word embeddings are high-dimensional, the predictions of trained models lie close to a low-dimensional manifold. Leveraging this property, we construct lightweight models in the reduced space to perform seen/unseen discrimination, enabling explicit bias correction.

% \begin{equation}
% \dtrain =
% \{(x_i, w_i) \mid x_i \in \mathcal{X},\ w_i \in \mathcal{W}_s\}
% \end{equation}
%(where $w_i=w(x_i)$) while the test dataset contains samples from both seen and unseen classes

% \begin{equation}
% \dtest =
% \{(x_j, w_j) \mid x_j \in \mathcal{X},\ w_j \in \mathcal{W}_s \cup \mathcal{W}_u\},
% \end{equation}
% where $w_j=w(x_j)$.

%Our goal is to learn a function $f_\theta : \mathcal{X} \rightarrow \mathbb{R}^d$, that maps a word image to a structured word embedding.
% \begin{equation}
% f_\theta : \mathcal{X} \rightarrow \mathbb{R}^d
% \end{equation}

% \begin{equation}
% \hat{w} =
% \arg\min_{w \in \mathcal{W}}
% \| f_\theta(x) - e(w) \|
% \end{equation}

%To address bias toward seen classes we additionally estimate whether the input belongs to a seen or unseen class.

\section{Problem Formulation and Methodology}

\textbf{Problem Formulation.} Let $\mathcal{X}$ denote a space of handwritten word images and $\mathcal{W}$ denote the corresponding vocabulary of word classes. Let $\mathcal{X} \xrightarrow{w} \mathcal{W}:x\mapsto w(x)$ denote a map taking the word image to the word class. Each word class $w \in \mathcal{W}$ is associated with an embedding vector $e(w) \in \mathbb{R}^d$. 

We partition the vocabulary into two disjoint sets $\mathcal{W} = \mathcal{W}_s \cup \mathcal{W}_u$,
where $\mathcal{W}_s$ represents the set of word classes available during training and $\mathcal{W}_u$ represents unseen word classes that appear only during testing. The data points (in the form of handwritten word images) are envisaged to be generated from a distribution $P$ on the space $\mathcal{X}$. 
The training and test datasets are defined respectively as 
\begin{center}
$\dtrain =
\{(x_i, w_i \mid x_i \in \mathcal{X},\ w_i=w(x_i) \in \mathcal{W}_s\};$ \\%\quad 
$\dtest =
\{(x_j, w_j) \mid x_j \in \mathcal{X},\ w_j \in \mathcal{W}_s \cup \mathcal{W}_u\}.$
\end{center}

We learn a function $f_\theta : \mathcal{X} \rightarrow \mathbb{R}^d$, that maps a word image to a structured word embedding; suitably parametrized by $\theta$ which may symbolize, for instance, the weights of a neural network which may drive the predictor $f_\theta$. Classically, recognition of a word is performed by nearest neighbor matching; in other words, $\hat{w}(x) =\arg\min_{w \in \mathcal{W}} \| f_\theta(x) - e(w) \|$.

\textbf{Methodology.} Our framework consists of two phases: a first phase consisting of $L$ copies of embedding models $(\monek{k})_{k=1}^L$, and a second phase comprising of seen/unseen classifiers  $\bigl(\bigl[\mtwoki{k}{i}\bigr]_{i=1}^{n_k}\bigr)_{k=1}^L$ ($n_k$ copies for each $\monek{k}$ from the first phase).

%\subsection{Embedding Models (\texorpdfstring{$\monek{k}$}{M1})}
{\bf Phase I: Embedding Models (\texorpdfstring{$\monek{k}$}{M1}).} Each embedding model learns a mapping $\monek{k} : \mathcal{X} \rightarrow \mathbb{R}^d$ that converts a word image into a structured embedding vector. The embedding encodes compositional properties of words including local letter shape descriptors, segment-level character presence and bigram statistics across different parts of a word. For each $1\le k \le L$, the model $\monek{k}$ is trained on a subset of the seen vocabulary $\mathcal{W}_s^{(k)} \subset \mathcal{W}_s$, while the remaining classes (i.e. $\mathcal{W}_s \cup \mathcal{W}_u \setminus \mathcal{W}_s^{(k)}$) are treated as unseen for the $k$-th model $\monek{k}$ from Phase I.
% \begin{equation}
% \mathcal{W}_u^{(k)} = \mathcal{W} \setminus \mathcal{W}_s^{(k)}
% \end{equation}
Training $1\le k \le L$ such models produces a diverse ensemble where each model has partial knowledge of the vocabulary.

%\subsection{Seen/Unseen Classifier (\texorpdfstring{$\mtwoki{k}{i}$}{M2deff})}
{\bf Phase II: Seen/Unseen Classifier (\texorpdfstring{$\mtwoki{k}{i}$}{M2deff}).}
Embedding models tend to be biased toward classes observed during training. In other words, the mis-identification error for the training status of a data point $x$ (generated from the distribution $P$), captured by the quantity $\P[\hat{w}(x)\in \mathcal{W}_s^{(k)} | w(x) \in \mathcal{W}_u^{(k)} ]$ is high (for comparison, in a perfect classification mechanism, this probability should be 0).
We define the indicator random variables $Z_k(x):=\ind [\{w(x)\in \mathcal{W}_s^{(k)}\}]$ and $\hat{Z}_k(x):=\ind [\{\monek{k}(x)\in \mathcal{W}_s^{(k)}\}]$. 
%In these terms, the mis-identification of training class may be summarized as $p_0(x):=\P[\hat{Z}_k(x)=1 \mid Z_k(x)=0]$.
%Similarly, we define  $p_1(x):=\P[\hat{Z}_k(x)=1 \mid Z_k(x)=1]$

To mitigate bias, in Phase II we introduce an ensemble of $n_k$ classifiers (for each fixed $\monek{k}$ from Phase I). At the fine-grained level, each $\left(\mtwoki{k}{i}\right)_{i=1}^{n_k}$ is trained to learn the training category $Z_k(x)$, with a 0-1 valued output $\hat{Z}_{k,i}(x)$. The $\mtwoki{k}{i}$-s are typically trained using a (random) subset of the training data $\mathcal{W}_s$ (say, around 80\%). So their 0-1 valued outputs will be different and will embody their inherent randomness. In an unbiased scenario, a natural way to aggregate the $\hat{Z}_{k,i}(x)$ would be to take a majority vote among the $\mtwoki{k}{i}$-s in order to determine the training status seen/unseen. But in view of the inherent bias of $Z_k$ towards being 1, we compute declare the training status as ``seen'' only  if the  total output $\left(\sum_{i=1}^{n_k}\hat{Z}_{k,i}(x)\right)$ crosses a threshold \texttt{thres} (to be thought of as significantly higher than $n_k/2$). In practice, the choice of \texttt{thres} can reasonably be carried out via cross validation.

% Via a Monte Carlo approach, using training data points from $\mathcal{W}_s \setminus \mathcal{W}_s^{(k)}$ and $\mathcal{W}_s^{(k)}$, it obtains an estimate $\hat{p}_0$ and $\hat{p}_1$ for $p_0$ and $p_1$ respectively.  
{\bf Low Dimensional Approximation.}
For practical purposes, we can additionally modify the the scheme outlined above by using a low dimensional projection $\monek{k}(x)$, via a standard dimension reduction methodology, such as {\it Principal Component Analysis} (PCA), and use this as an input to the Phase II classifiers $\left(\Pi[\mtwoki{k}{i}]\right)_{i=1}^{n_k}$. The ability to operate on a low-dimensional problem makes each $\mtwoki{k}{i}$ to be computationally light, which entails that we are able to deploy a very large number $n_k$ of such classifiers leading to improved accuracy of the Monte Carlo method for estimation of training status. The trade-off is the approximation incurred due to working with the low-dimensional representation $\Pi[\monek{k}(x)]$ instead of the full word embedding $\monek{k}(x)$.

% \begin{equation}
% z_k(x) = \Pi\!\left(\monek{k}(x)\right)
% \end{equation}
%notations everywhere
%M_{[1]}^{(k)}

%$M_{[2],i}^{(k)}$ % drop the dimension

%where $z \in \mathbb{R}^{d_\mathrm{eff}}$ is obtained using some standard dimension reduction methodology, such as PCA.

% For each $k \in \{1,\ldots,L\}$ from Phase I, the ensemble of $n_k$ Phase II classifiers $\bigl[\mtwoki{k}{i}\bigr]_{i=1}^{n_k}$ estimates the probability of 

% \begin{equation}
% h_k(z_k(x)) = \mathbb{P}(w(x) \in \mathcal{W}_s^{(k)} \mid z_k(x))
% \end{equation}
% \textcolor{red}{Introduce some noiseless model for this probability to make sense}
% Multiple such classifiers are trained to form an ensemble for robust prediction.

\textbf{Inference.}
Given an input word image $x$, we will compute a prediction for the word $w(x)$ for each $1\le k \le L$, and subsequently aggregate these predictions. For a fixed $k$, the embedding model $\monek{k}$ in Phase I produces $
v_k(x) := \monek{k}(x)$. The $\left(\mtwoki{k}{i}\right)_{i=1}^{n_k}$ ensemble predicts whether the embedding corresponds to a seen or unseen class (for that Phase I model $\monek{k}$. We denote this predicted training status as $\hat{T}^{(k)}$, which can take two values -- seen/unseen. 
If $\hat{T}^{(k)}=\text{seen}$ (resp., unseen), then the predicted word is the ``best approximation'' to $\monek{k}(x)$ from the set $\mathcal{W}_s^{(k)}$ (resp. from the set $\mathcal{W} \setminus \mathcal{W}_s^{(k)}$).

The ``best approximation'' above can be computed via nearest neighbour matching within the relevant subset of the vocabulary. For the seen class $\mathcal{W}_s^{(k)}$. This can optionally be substituted by a random forest or logistic regression based classifier trained on the corresponding data.

Predictions for $1\le k \le L$ are aggregated (via most frequent prediction, with ties being broken via nearest neighbour matching)  to obtain the final word prediction.

\section{The Algorithmic Pipeline} % TODO change name of section

\begin{figure}
\centering
\begin{tikzpicture}[
    >=Stealth,
    node distance=1.0cm,
    every node/.style={font=\small},
    algo/.style={
        draw=orange!50,
        fill=yellow!30,
        rounded corners,
        thick,
        align=center
    },
    optalgo/.style={
        draw=green!50!black,
        fill=green!30,
        rounded corners,
        thick,
        align=center
    },
    %{draw, rounded corners, thick, fill=green!20, minimum width=2.5cm, minimum height=1cm},
    circ/.style={
        draw=red!50!black,
        fill=red!10,
        circle,
        minimum size=1cm,
        align=center
    },
    label/.style={font=\footnotesize, align=center}
]

% Nodes
\node[circ] (D1) {Data};
\node[algo, right=0.3cm of D1] (A1) {Algorithm \ref{alg1}};
\node[circ, right=0.3cm of A1] (D2) {$\monek{k},$\\$T,X,$\\$X'$};
\node[algo, right=0.3cm of D2] (A2) {Algorithm \ref{alg2}};
\node[circ, right=0.3cm of A2] (D3) {$\hat{X},\mtwoki{k}{i}$};
\node[algo, right=0.3cm of D3] (A3) {Algorithm \ref{alg3}};
\node[circ, below=1.3cm of D1] (D4) {$X_\text{seen},$\\$\mathrm{RF},$\\$\mathrm{LR}$};
\node[algo, right=0.3cm of D4] (A5) {Algorithm \ref{alg5}};
\node[circ, right=0.3cm of A5] (P1) {Predictions};
\node[algo, right=0.3cm of P1] (A6) {Algorithm \ref{alg6}};
\node[circ, right=0.3cm of A6] (P2) {Final\\Predictions};
% Connections
\draw[->] (D1) -- (A1);
\draw[->] (A1) -- (D2);
\draw[->] (D2) -- (A2);
\draw[->] (A2) -- (D3);
\draw[->] (D3) -- (A3);
\draw[->] (A3) -- ++(1.6,0) -- ++(0,-1.4) -- ++(-11.45,0) -- (D4);
\draw[->] (D4) -- (A5);
\draw[->] (A5) -- (P1);
\draw[->] (P1) -- (A6);
\draw[->] (A6) -- (P2);

% Labels
\node[label, below=0.2cm of A1] {Split data};
\node[label, below=0.2cm of A2] {Synthetic data for $\mtwoki{k}{-}$};
\node[label, below=0.2cm of A3] {Synthetic data for RF/LR};
\node[label, above=0.1cm of A3] {(Optional)};
\node[label, below=0.2cm of A5] {Prediction per $\monek{k}$};
\node[label, below=0.2cm of A6] {Combining predictions};
\end{tikzpicture}
\caption{Schematic of the algorithmic pipeline for predicting words for IAM handwriting dataset}
\end{figure}
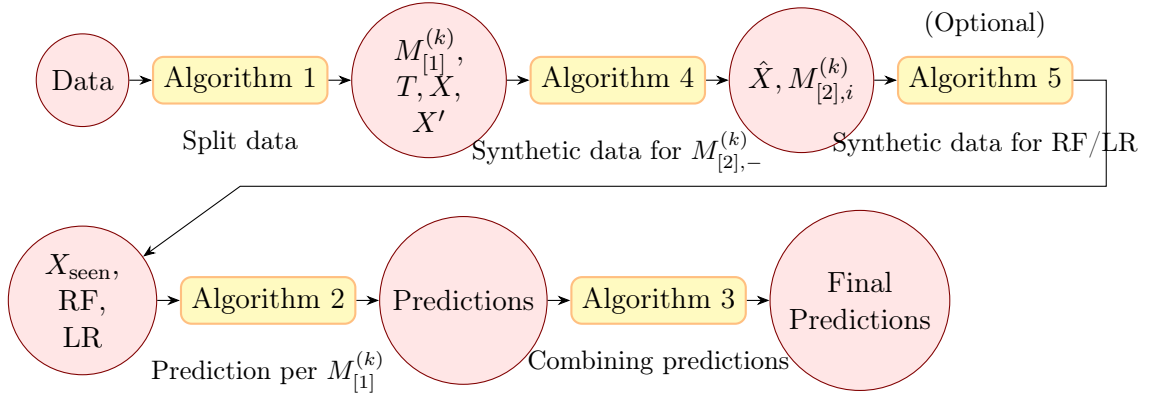
%Seen Unseen Prediction For Each Split
% Seen Unseen Classifier Classifying Embedding-Model Split Samples

% let's try the list here
We train the entire system of predicting $w(x)$ in these stages:
\begin{balditemize}
\item \emph{Training embedding model $\monek{k}$.} Algorithm \ref{alg1} does sample splitting by taking a subset of the 250-by-50 pixel images $x$ with $w(x)\in\mathcal{W}_s^{(k)}$. These splits are used to train $\monek{k}$. For the values of the splitting fractions, we typically take $s\approx0.1$ and $t\approx0.7$. % we typically take around $10\%$ and $70\%$ respectively.
\item \emph{Training ensemble seen/unseen classifier $\mtwoki{k}{-}$.} Algorithm \ref{alg2} (c.f. Section \ref{app:sec:other-alg} in the Appendix) is used to build a dataset to train one ensemble of $\mtwoki{k}{-}$ models for each $k$.  The dataset includes two parts: synthetic data `far' from the seen classes as unseen examples, and remaining training data ($X,X'$) not used for training $\monek{k}$. Each $\mtwoki{k}{i}$ is a is a neural network that is trained with 80\% of the dataset, so that each $\mtwoki{k}{i}$ in the ensemble produces meaningfully different results that can be combined. This ensemble predicts if the $\monek{k}$ has been trained on the word, i.e. $w(x)\in\mathcal{W}_s^{(k)}$, given $\monek{k}(x)$.
%, for mitigating the large difference in accuracy between seen and unseen words.
\item \emph{(Optional) Training word predictor.} Algorithm \ref{alg3} (c.f. Section \ref{app:sec:other-alg} in the Appendix) produces a dataset with two parts: training data with words in $\mathcal{W}_s^{(k)}$ that were not used to train $\monek{k}$, and synthetic data. This is used
for training a random forest (RF) or logistic regression (LR) model that could predict the word among $\mathcal{W}_s^{(k)}$ (RF or LR could replace nearest-neighbour matching).
\item \emph{Word prediction from single $\monek{k}$.} Algorithm \ref{alg5} produces a prediction for each $\monek{k}$ for the images in the test set. We use $\monek{k}(x)$ to get the estimated feature vector, then use $\mtwoki{k}{-}$ to predict whether $w(x)\in\mathcal{W}_s^{(k)}$. If it is predicted that $w(x)\in\mathcal{W}_s^{(k)}$, then Algorithm \ref{alg4} (c.f. Section \ref{app:sec:other-alg} in the Appendix) does nearest-neighbour matching among $\mathcal{W}_s^{(k)}$ (or uses the random forest/logistic regression) to give a prediction for the word. If it If is predicted that $w(x)\in \mathcal{W}\setminus \mathcal{W}_s^{(k)}$, then nearest-neighbour matching among $\mathcal{W}\setminus\mathcal{W}_s^{(k)}$ is used.
\item \emph{Aggregate word prediction from all $\monek{k}$.} Algorithm \ref{alg6} takes the prediction from each $\monek{k}$ to form a final prediction. Algorithm \ref{alg6} does this by taking the most frequent prediction, breaking ties by nearest-neighbour matching of $\monek{k}(x)$.
\end{balditemize}

%previous

To train the $\monek{k}$s in Algorithm \ref{alg1}, we use the same PhoS and PhoC vectors given in the original Pho(SC)Net approach \cite{bhatt2022phoscctc}. In particular, for each of 8407 words, we form a 769-dimensional word embedding vector as follows:
\begin{balditemize}
\item Shapes: The word is split into 1 to 5 parts (for 15 parts in total), each with 11 dimensions representing the shapes of the letters, for a total of 165 dimensions
\item Letters: The word is split into 2 to 5 parts (for 14 parts in total), each with 36 dimensions representing letters appearing in that part, for a total of 504 dimensions
\item Bigrams: The word is split into halves, each with 50 dimensions representing bigrams, for a total of 100 dimensions
\end{balditemize}

\begin{algorithm}[!t]
\caption{Producing dataset splits and training $\monek{k}$}
\label{alg1}
\hspace*{\algorithmicindent} \textbf{Input:} $\dtrain$, the training data; \quad $s, t\in(0,1)$, the data splitting fractions. \\
\hspace*{\algorithmicindent} \textbf{Output:} $\monek{k}, T, X, X'$, where $T$ represents the reduced training data for $\monek{k}$, $X$ will be the `seen' training data for $\mtwoki{k}{i}$, and $X'$ will be the `unseen' training data for $\mtwoki{k}{i}$
\begin{algorithmic}[1]
    \State In $\dtrain$, we randomly select the fraction $s$ of the labels to be denoted as unseen, then filter out the training data with such labels and name the resulting dataframe $X'$. 
    \State We further split the remaining $\dtrain\backslash X'$, into $t:(1-t)$ based on per-class splitting, naming the fraction $t$ split as $T$, and the fraction $1-t$ split as $X$.
    \State Train $\monek{k}$ on $T$.
    %\State Using the trained $\mathsf{M}^{(1)}$, predict the feature vectors of $X \cup X'$. The instances of $X$ are labelled as 1 and the instances of $X'$ is labelled as 0 for seen and unseen training of $\mathsf{M}^{(2)}$.
    %\State We then train 50 different $\mathsf{M}^{(2)}$ on the predicted features on a 80\%-20\% split that is done randomly, where 20\% of the data is used as validation cases to measure the performance of the $\mathsf{M}^{(2)}$. We then append each $\mathsf{M}^{(2)}$ to $\mathcal{M}^{(2)}$ which ultimately consists of all 50 $\mathsf{M}^{(2)}$s.
    \State return $\monek{k},T,X,X'$.
\end{algorithmic}
\end{algorithm}

\begin{algorithm}[!t]
\caption{Predicting classes for each $\monek{k}$}
\label{alg5}
\hspace*{\algorithmicindent} \textbf{Input:}  $\mathfrak{d}$ a test datapoint, $\monek{k},\mtwoki{k}{-},\mathcal{W}_s^{(k)},\mathcal{W}_u^{(k)}$, \texttt{thres} a threshold \\
\hspace*{\algorithmicindent} \textbf{Output:} \texttt{Predictions}[$k$] and $F[k]$, prediction for  word class and feature vector (resp.) for $\mathfrak{d}$ from $\monek{k}$; 
\begin{algorithmic}[1]
%    \For {each $\mathfrak{d}[i]$ of $\mathfrak{d}$}
        \State $F[k]\gets\monek{k}(\mathfrak{d})$ \Comment{Predicted feature vector}
        \State \texttt{count} $\gets$ number of $\mtwoki{k}{-}$ that predict $F$ corresponds to a `seen' class (i.e., in $\mathcal{W}_s^{(k)}$)
        %\State count $\gets\sum_i\mtwoki{k}{i}\circ\Pi(F)$ \Comment{$\mtwoki{k}{i}$ generally predicts from a projection of $F$}
        \If{\texttt{count} $\geq$ \texttt{thres}}
            \State Use Algorithm \ref{alg4} to predict class of $F$ among $\mathcal{W}_s^{(k)}$; store class to \texttt{Predictions}[$k$]
        \Else
            \State \texttt{Predictions}$[k]$ $\gets$ nearest distance class from $\mathcal{W}_u^{(k)}$ \Comment{Predicted word class}
        \EndIf 
    %\EndFor\\
    \State return \texttt{Predictions}[$k$], $F[k]$
\end{algorithmic}
\end{algorithm}

\begin{algorithm}[!t]
\caption{Combining predictions}
\label{alg6}
\hspace*{\algorithmicindent} \textbf{Input:} $\mathfrak{d}$ a test datapoint, all the $F[-]$, \texttt{Predictions}[$-$] (via Alg. \ref{alg5}) \\
\hspace*{\algorithmicindent} \textbf{Output:} \texttt{Final\_Prediction}, a finalized word prediction for $\mathfrak{d}$
\begin{algorithmic}[1]
%    \For {each $\mathfrak{d}[i]$ of $\mathfrak{d}$}
        \State List the most frequent labels among \texttt{Predictions}[$-$]
        \If {there is a unique most frequent label}
        \State \texttt{Final\_Prediction} $\gets$ the unique most frequent label
        \Else
            \State For each of the most frequent labels $c=$\texttt{Predictions}[$l$], compute $\|e(c) - F[l]\|$ 
            %and $\monek{-}(\dtest[i])$, to get a distance for each $\monek{k}$.
            \State Let $c^* = \mathrm{argmin}_c \|e(c) - F[l]\|$ in the above step  
            \State \texttt{Final\_Prediction} $\gets c^*$
        \EndIf
    %\EndFor
    \State return \texttt{Final\_Prediction}
\end{algorithmic}
\end{algorithm}

\section{Results \& Discussion}
\label{sec:results}
%quick problem setup handwriting data
The IAM handwriting data \cite{Marti2002} consists of $250$ pixel by $50$ pixel images of handwritten words, each which are one of 8407 words. 7898 of the words are given in the training set, while the other 509 words are kept as unseen classes. The task is to determine which of 8407 words is written.

% describe in sentences this table
We store the accuracy on unseen classes $A_u$, accuracy on seen classes $A_s$. We compute the {\it harmonic mean} $h:=\frac{2A_sA_u}{A_s+A_u}$ to measure performance at both prediction tasks. As an alternate metric of effectiveness, we also investigate the {\it balanced accuracy} $A_b := \min(A_u, A_s)$.

%\textcolor{red}{TODO short paragraph for each method + description on performance}

The original Pho(SC)Net approach involved training a model to predict the feature vector for image handwriting samples. As the table shows, there is a bias towards predicting seen classes, with significantly lower unseen class accuracy (0.77) than seen class accuracy (0.93).

The other experiments are improvements over the original Pho(SC)Net approach via the paradigm of using many small models $\mtwoki{k}{-}$ to determine if the feature vector corresponds to a seen class $\mathcal{W}_s^{(k)}$ or an unseen class $\mathcal{W}_u^{(k)}$. Afterwards, a threshold count `\texttt{thres}' (in Algorithm \ref{alg5}) is chosen. If there are more models which report the feature vector as seen, compared to `\texttt{thres}', the feature vector is considered seen, and prediction will only be done among seen classes. `\texttt{thres}' was roughly chosen to maximize $h$, the harmonic mean of the unseen and seen accuracy. %If the threshold is too low, this would increase the seen accuracy and decrease the unseen accuracy. Likewise, if the threshold is too high, this would decrease the seen accuracy and increase the unseen accuracy.

%Otherwise, the feature vector is considered unseen, and prediction will only be done among unseen classes.
\begin{table}[!h]
    \centering
    \caption{Implementation details}
    \label{tab:impl-details}
    \begin{tabular}{lccccc}
        \toprule
        & $L$
        & $n_k$
        & Dim. for $M_{[2]}$
        & \texttt{thres} for $h$
        & \texttt{thres} for $A_b$\\
        \midrule
        $M_{[1]}$ (Original Pho(SC)Net) \cite{bhatt2022phoscctc}
        &1&$-$&$-$&$-$&$-$\\
        (21)$M_{[1]} + 50 M_{[2]}$ (769 dim, NN)
        &21&50&769&49&49\\
        (5)$M_{[1]} + 500 M_{[2]}$ (30 dim, NN) &5&500&30&445&430\\
        (5)$M_{[1]} + 500 M_{[2]}$ (15 dim, NN) &5&500&15&440&400\\
        (5)$M_{[1]} + 500 M_{[2]}$ (15 dim, RF) &5&500&15&425&405\\
        (5)$M_{[1]} + 500 M_{[2]}$ (15 dim, LR) &5&500&15&435&420\\
        %\hspace{1em}(10)$\mathsf{M1}$ only&0.831&0.952&0.887\\
        \bottomrule
    \end{tabular}
\end{table}
\begin{table}[!h]
    \centering
    \caption{Experimental results on handwritten test dataset}
    \label{tab:Test results-phoscnet}
    \begin{tabular}{lcccc}
        \toprule
        & $A_u$ & $A_s$ & $h$ & $A_b$\\
        \midrule
        % \textbf{IAM} & & & & \\
        \hspace{1em}$M_{[1]}$ (Original Pho(SC)Net) \cite{bhatt2022phoscctc} &0.77  &\textbf{0.93} &0.84 & $-$\\
        %\hspace{1em}(21)$M_{[1]} + 50 M_{[2]}$(769 dimensions, NN) &\textbf{0.907} &0.906 &\textbf{0.906} & \textbf{0.906}\\
        \hspace{1em}(21)$M_{[1]} + 50 M_{[2]}$ (769 dim, NN) &\textbf{0.931} &0.903 &\textbf{0.917} & \textbf{0.903}\\
        \hspace{1em}(5)$M_{[1]} + 500 M_{[2]}$ (30 dim, NN) &0.916&0.857&0.886&0.877
        \\
        \hspace{1em}(5)$M_{[1]} + 500 M_{[2]}$ (15 dim, NN) &0.905&0.850&0.877&0.864
        \\
        \hspace{1em}(5)$M_{[1]} + 500 M_{[2]}$ (15 dim, RF) &0.905&0.847&0.875&0.867
        \\
        \hspace{1em}(5)$M_{[1]} + 500 M_{[2]}$ (15 dim, LR) &0.900&0.829&0.863&0.847
        \\
        %\hspace{1em}(10)$\mathsf{M1}$ only&0.831&0.952&0.887\\
        \bottomrule
    \end{tabular}
\end{table}

We consider class selection methods: nearest-neighbour (NN), random forest (RF) and logistic regression (LR).

(21)$M_{[1]} + 50 M_{[2]}$ (769 dim, NN) is trained without dimension reduction, and achieves the best performance, with the highest $h=0.917$. Each $\mtwoki{k}{i}$ has 10 ReLU hidden layers of 256 nodes.

%The other methods involve dimension reduction to 15 dimensions. By reducing to 15 dimensions, more seen-unseen classifiers could be trained cheaply.

\begin{figure}[!h]
    \centering
    % TODO increase axis label sizes
    \includegraphics[width=0.4\linewidth]{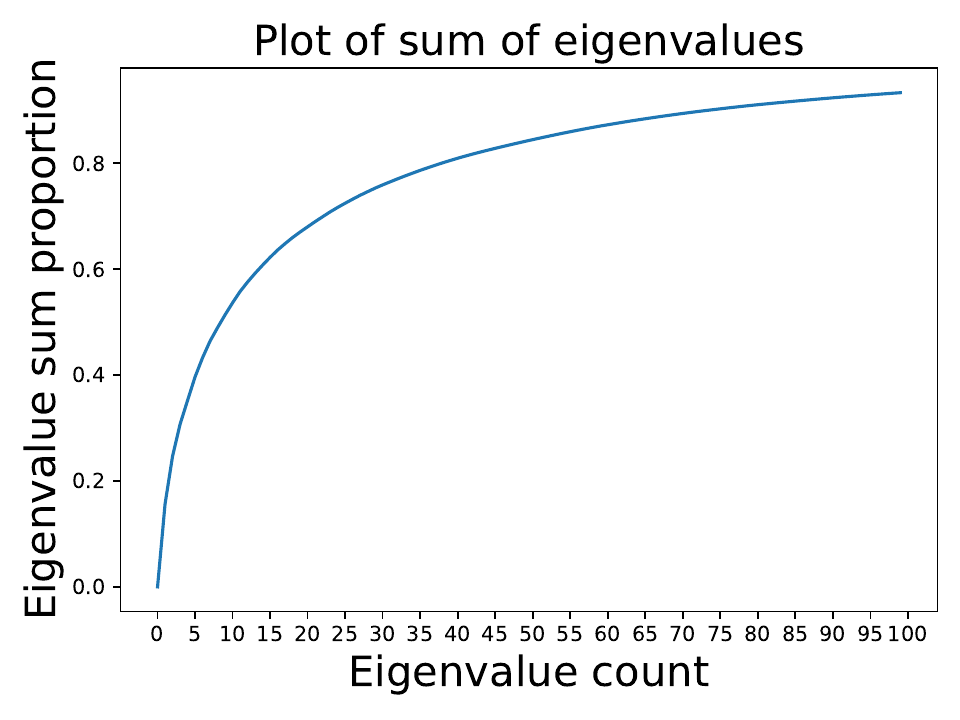}
    \includegraphics[width=0.4\linewidth]{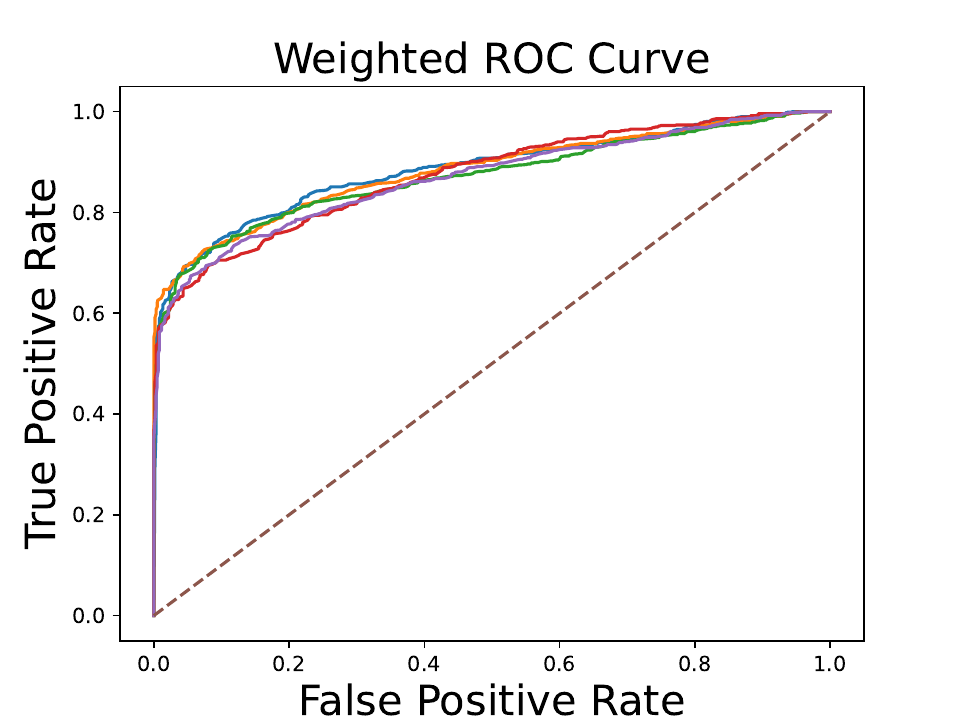}
    \caption{Plot of cumulative sum of eigenvalues, ROCs for ensemble $\mtwoki{k}{-}$ model}
    %15,0.6224002429095676
    %30,0.7592010767237886
    \label{fig:eigenplot}
\end{figure}

We can consider the covariance matrix of the class embedding vectors. We see that the top 15 eigenvalues account for 0.6224 of the sum of all eigenvalues while the top 30 eigenvalues account for 0.7592 of the sum of all eigenvalues. Figure \ref{fig:eigenplot} shows the the sum of largest eigenvalues of the covariance matrix, as a fraction the sum of all the eigenvalues of the covariance matrix.

The ROCs for the ensemble $\mtwoki{k}{-}$ model (for each $k$) are computed by weighting the test set, to give equal total weight to the tests where $w(x)\in\mathcal{W}_s$, and the tests where $w(x)\in\mathcal{W}_u$, of which there are a significantly different number of both tests. The ensembles used here were of 500 30-dimensional $\mtwoki{k}{i}$. We can see the 5 different ensembles have very similar curves.
%\textcolor{red}{TODO Weighted ROC, + explain weighting (datasets different size)}

%Details about the other methods
For both (5)$M_{[1]} + 500 M_{[2]}$ (15 dim, NN) and (5)$M_{[1]} + 500 M_{[2]}$ (30 dim, NN), each $\mtwoki{k}{i}$ has 10 ReLU hidden layers of 32 nodes.

(5)$M_{[1]} + 500 M_{[2]}$ (15 dim, RF). It is the same as (5)$M_{[1]} + 500 M_{[2]}$ (15 dim, NN), but we add a RF, trained on 15-dimensional representations with scikit-learn with a depth of 10. The seen RF turns out to be 5GB per M1. One limitation is that the computational cost is intensive in the height of the tree.

(5)$M_{[1]} + 500 M_{[2]}$ (15 dim, LR). It is the same as (5)$M_{[1]} + 500 M_{[2]}$ (15 dim, NN), but we add a LR, trained on 15-dimensional representations with scikit-learn with 1000 iterations. One limitation is numerical stability.

% \textcolor{red}{TODO Conclusions about the table in one paragraph}
 All the variations of our proposed methodology provide a {\it significant improvement} over the Pho(SC)Net (c.f. Table~\ref{tab:Test results-phoscnet}). In particular, for the crucial problem of  unseen class detection, we are able achieve an accuracy improvement of 0.931 (from 0.77). The overall GZSL accuracy ($h$) also improves 0.917 from 0.84. While
 the (21)$M_{[1]} + 50 M_{[2]}$ ensemble approach is expectedly the most accurate ($h=0.917$), the other approaches involving dimension reduction are also much more accurate than the original Pho(SC)Net. Thus, low-dimensional ensemble approaches provide a computationally lighter alternative to reduce bias for practical applications in various GZSL problems.

\textbf{Note:} All the data analyses reported in this section were conducted on a shared compute cluster with 1 NVIDIA A40 (48GB), 36 cores Intel 9452Y CPUs and 256GB RAM. The most expensive computation, (21)$M_{[1]} + 50 M_{[2]}$ (769 dimensions, NN) took approximately a week of compute, whereas the other experiments can be run in 3 days, for a total of roughly 3 weeks of compute.

%hardware, i.e. HPC

% stopping rules/runtime?
% defaults from scikit learn

\section{Theoretical Underpinnings of Ensembling for Bias Reduction}
In this section, we theoretically justify the idea of using an ensembling strategy for reducing the bias towards seen classes by considering a simplified model which is amenable to mathematical analysis. Suppose we have a dataset consisting of pairs $(x_i, w_i)$, where each $w_i \in \mathcal{W}_s$. Suppose we have $L$ binary predictors $\ensclass{k}$, $1\le k\le L$, where $\psi_k$ is trained on a subset of the available data by restricting to certain class labels $\mathcal{W}^{(k)}_s \subset \mathcal{W}_s$. Given an new observation $(x, w)$, $\ensclass{k}$ is supposed to identify whether the true class label $w$ of $x$ is from among the classes it has seen during its training. As such it gives binary outputs `seen' or `unseen'. We shall 
assume that $\mathcal{W}_s = \cup_{k = 1}^L \mathcal{W}_s^{(k)}$. Then $\mathcal{W}_u=\mathcal{W}\setminus\mathcal{W}_s$ are the common unseen classes. We note here that by sample splitting, one may ensure that the predictors $\psi_k$ are independent (albeit at the loss of some efficiency). In the supplementary material, we analyse the case where the $\psi_k$'s are only assumed to be exchangeable.

A simple example of such binary predictors would be as follows. For each $k$, we fit a Gaussian mixture model to the corresponding training data. Denoting by $\hat{\phi}_k$ the fitted the probability density function, we declare $\ensclass{k}(x)=\text{`unseen'}$ if $\hat{\phi}_k$ has a `low' value at $x$.

Let us introduce shorthands for the `Type 1' and `Type 2' errors of the binary predictors:
\[
    \alpha_k := \mathbb{P}(\ensclass{k}(x)=\text{`unseen'} \mid w\in\mathcal{W}_s^{(k)}), \qquad
    \beta_k := \mathbb{P}(\ensclass{k}(x)=\text{`seen'} \mid w\not\in\mathcal{W}_s^{(k)}).
\]
We shall assume that $\alpha_k +\beta_k \le 1$. Indeed, a random predictor has $\alpha_k+\beta_k=1$, so we are assuming the predictor $\psi_k$ to be no worse than random. %Indeed, if $\alpha_k+\beta_k > 1$, then by flipping the predictions of $\psi_k$, one gets a predictor whose Type 1 and Type 2 errors add up to $< 1$.

Now let $E_m$ be the ensemble predictor that outputs `seen' if and only if at least $m$ of the classifiers $\ensclass{k}$ output `seen'. We shall show that for appropriate choices of $m$, $E_m$ achieves high accuracy in classifying both `seen' and `unseen' classes.

We shall present our results on the accuracy of $E_m$ under two different \emph{class-allocation} mechanisms.
\begin{balditemize}
    \item [(i)] \,\, \emph{Arbitrary allocation.} Seen classes are assigned to $\mathcal{W}_{s}^{(k)}$ in some arbitrary manner.
    \item [(ii)] \,\, \emph{Random allocation.} Each seen class $w \in \mathcal{W}_s$ is included in $\mathcal{W}_s^{(k)}$ with probability $p > 0$, and this procedure is repeated independently across $k$.
\end{balditemize}
We define a few key quantities to state our results. For $w' \in \mathcal{W}_s$, let $\mathcal{L}(w'):=\{k \mid w' \in\mathcal{W}_s^{(k)}\}$. Let $\gamma_k := 1 - \alpha_k - \beta_k \ge 0$, $\Delta(w') := \sum_{k \in \mathcal{L}(w')} \gamma_k$ and $\Delta := \min_{w' \in \mathcal{W}_s} \Delta(w')$.
\begin{theorem}[Arbitrary allocation]\label{thm:arbit}
Assume that $\gamma_k > 0$ for all $k\in[L]$, so that $\Delta > 0$, Then, for $m=\lceil \mu(1+\delta)\rceil$, where $\mu=\sum_{k\in[L]}\beta_k$ and $\delta=\frac{\Delta}{2\mu+\Delta}$,  we have
\begin{center}
    $\mathbb{P}(E_m(x) = \text{`unseen'} \mid w \in \mathcal{W}_s) \le e^{-\frac{\delta^2 (\mu + \Delta)}{2}},$\\
    $\mathbb{P}(E_m(x) = \text{`seen'} \mid w \in \mathcal{W}_u) \le e^{-\frac{\delta^2 \mu}{2 + \delta}}.$
\end{center}
\end{theorem}

When $\Delta$ is bounded away from zero, and $\mu = \Theta(L)$, we see from Theorem~\ref{thm:arbit} that $E_m$ is able to classify both seen and unseen classes with high accuracy.

\begin{theorem}[Random allocation]\label{thm:random}
Let $\Delta_0 > 0$ be a deterministic threshold.  
For $m = \lceil \mu (1 + \delta)\rceil$, where $\mu = \sum_{k \in [L]} \beta_k$ and $\delta = \frac{\Delta_0}{2\mu + \Delta_0}$, we have
\begin{align*}
\mathbb{P}(E_m(x) = \text{`unseen'} \mid w \in \mathcal{W}_s) &\le \mathbb{P}(\Delta < \Delta_0) +   \mathbb{E}\big[e^{-\delta^2 (\mu + \Delta_0) / 2}\big],\\
\mathbb{P}(E_m(x) = \text{`seen'} \mid w \in \mathcal{W}_u) &\le
\mathbb{P}(\Delta < \Delta_0) +   \mathbb{E}\big[ e^{-\delta^2\mu/(2+\delta)}\big].
\end{align*}
If we further assume that the $\alpha_k$'s and the $\beta_k$'s are statistically independent of the class allocation mechanism, then, with $\eta \in (0, 1)$ and $\Delta_0 := (1 - \eta) p \sum_{k \in [L]} \gamma_k$, we have
\begin{align*}
\mathbb{P}(E_m(x) = \text{`unseen'} \mid w \in \mathcal{W}_s) &\le |\mathcal{W}_s| e^{-\frac{2 \eta^2 p^2 (\sum_{k \in [L]} \gamma_k)^2}{\sum_{k \in [L]} \gamma_k^2}} + e^{-\delta^2 (\mu + \Delta_0) / 2},\\
\mathbb{P}(E_m(x) = \text{`seen'} \mid w \in \mathcal{W}_u) &\le |\mathcal{W}_s| e^{-\frac{2 \eta^2 p^2 (\sum_{k \in [L]} \gamma_k)^2}{\sum_{k \in [L]} \gamma_k^2}} + e^{-\delta^2\mu/(2+\delta)}.
\end{align*}
\end{theorem}

We see that random allocations, assuming independence of the $\alpha_k$'s and the $\beta_k$'s from the allocation scheme, the ensemble classifier $E_m$ is able to identify both seen and unseen classes with high accuracy, if $\mu$ is large, $\delta$ is bounded away from $0$, and $\frac{(\sum_{k \in [L]} \gamma_k)^2}{\sum_{k \in [L]} \gamma_k^2}$ is large. For instance, if we assume a uniform lower bound $\gamma_k = 1 - \alpha_k -\beta_k \ge \varepsilon > 0$, then $\Delta_0 \ge (1 - \eta) p \varepsilon L$ and $\frac{(\sum_{k = 1}^L \gamma_k)^2}{\sum_{k = 1}^L \gamma_k^2} \ge \varepsilon^2 L$.

We can see Algorithm $\ref{alg6}$ using multiple predictions to form a single prediction by taking the most common-occurring prediction. Each $\monek{k}$ is biased towards predicting words in $\mathcal{W}_s^{(k)}$, affecting the accuracy of words in $\mathcal{W}\setminus\mathcal{W}_s^{(k)}$, especially words not in the training set (in $\mathcal{W}_u$). In our approach, multiple $\monek{k}$ trained on different $\mathcal{W}_s^{(k)}$ give some robustness against an incorrect prediction for words inside $\mathcal{W}_u$.

\section{Concluding Remarks}\label{sec:conc}
In this work, we have devised a novel methodology for rectifying the inherent bias toward seen classes in genralised zero shot learning in the context of handwritten word recognition. Broadly speaking, it is a statistical methodology based on an ensembling strategy that leverages a two stage architecture, with the second stage explicitly focusing on bias mitigation. The proposed method is very general and can in principle be augmented to any classical GZSL learner in a turn-key fashion to implement bias correction. The present work is focused on the handwritten word recognition appplication, and extensions of this approach to other GZSL scenarios is a natural area for follow-up investigation. We provide a foundational theoretical structure to complement our methodology, and extensions of this theory to cover wider application domains and weaker hypothesis setups is another natural direction of research.

% \textcolor{red}{TODO add discussion about weighted nearest neighbours (Soumendu)}

% \textcolor{red}{Add limitations.}

\section{Acknowledgements}
The authors would like to acknowledge that computational work involved in this research work is partially supported by NUS IT's Research Computing group using grant number NUSREC-HPC-00001. SG was supported in part by the NUS Dean’s Chair Associate Professorship E-146-00-0037-01 and the
Singapore MOE grants A-8002014-00-00 and A-8003802-00-00.

\bibliographystyle{plain}
\bibliography{refs}

\appendix
\section{Supporting Algorithms}
\label{app:sec:other-alg}
Algorithm \ref{alg2} is used to generate unseen data for $\mtwoki{k}{-}$. This is run after Algorithm \ref{alg1} and before Algorithm \ref{alg5}. This uses the datasets $X,X'$ (which were not used to train $\monek{k}$) as well as additional synthetic data.
\begin{algorithm}[H] %H for exactly Here
\caption{Synthetically generating unseen data for $\mtwoki{k}{-}$}
\label{alg2}
\hspace*{\algorithmicindent} \textbf{Input:} $\monek{k}, X, X', L, R$ \Comment{$L$ and $R$ are the lower and upper bounds of the extremes respectively}\\
\hspace*{\algorithmicindent} \textbf{Output:} $\hat{X}$, 
consisting of feature vectors derived from $X\cup X'$ and synthetically generated data
\begin{algorithmic}[1]
    \State $N \gets$ length($X \cup X'$) \Comment{number of instances to generate} 
    \State dim $\gets$ \text{ size of feature vector}
    \State Initialize datasets $Z,Z'$
    \For {each $X_i$ of $X$}
        \State $Z[i]\gets\monek{k}(X_i)$ \Comment{$Z$ contains feature vectors of $X$, predicted by $\monek{k}$}
    \EndFor
    \For {each $X'_i$ of $X'$}
        \State $Z'[i]\gets\monek{k}(X'_i)$ \Comment{$Z'$ contains the feature vectors of $X'$, predicted by $\monek{k}$}
    \EndFor
    \State Calculate mean ($\mu$) and standard deviation ($\sigma$) for each coordinate of feature vectors in $Z$.
    \State $\hat{X'} \gets$ array($N$, dim) \newline \Comment{Store the feature vectors of the synthetically generated training unseen, used to train $\mtwoki{k}{-}$.}
    \For {each $X'[i]$ of $X'$}
        \State 95\% of the coordinates of $\hat{X'}[i]$ are given an {\it extreme sample} \newline \Comment from $\leq L$\%  quantile or $\geq R$\% quantile of $N(\mu_j, \sigma_j)$.
        \State The rest 5\% of coordinates of $\hat{X'}[i]$ are given {\it middle samples} \newline \Comment 
        from between $L\%$  quantile and  $R\%$ quantile of $N(\mu_j, \sigma_j)$.
    \EndFor
    \State return a dataframe $\hat{X}$ that consists of $Z, Z'$ and $\hat{X'}$, attaching $1$ to the feature vectors of $Z$, and $0$ to $Z', \hat{X'}$. \Comment{$1$ means it is seen, $0$ means it is unseen.}
\end{algorithmic}
\end{algorithm}
Algorithm \ref{alg3} (optional) is used to generate data for random forest/logistic regression approaches. This is run after Algorithm \ref{alg1} and before Algorithm \ref{alg5}. This uses the dataset $X$ (which was not used to train $\monek{k}$) as well as additional synthetic data.
\begin{algorithm}[H] %H for exactly Here
\caption{Synthetically generating seen data for random forest/logistic regression}
\label{alg3}
\hspace*{\algorithmicindent} \textbf{Input:} $\monek{k}$, $X$.\\
\hspace*{\algorithmicindent} \textbf{Output:} Dataset $X_\text{seen}$ for training random forest
\begin{algorithmic}[1]
\State $m\gets$ maximum number of training samples among $\mathcal{W}_s^{(k)}$ in $X$
\State $D$ stores differences between feature vectors and the corresponding class
\For{each $X_i$ in $X$}
    \State $F\gets\monek{k}(X_i)$
    \State Add feature vector $F$ and class $w(X_i)$ to $X_\text{seen}$
    \State Add the difference $F-e(w(X_i))$ to $D$
\EndFor
\For{each class $c\in\mathcal{W}_s^{(k)}$}
    \While{class has $<m$ samples}
        \State $v\gets$ Uniform sample from $D$
        \State Add $v+e(c)$ and class to $X_\text{seen}$.
    \EndWhile
\EndFor
\State return $X_\text{seen}$
\end{algorithmic}
\end{algorithm}

Algorithm \ref{alg4} is used to apply the random forest/logistic regression (if available) otherwise it applies the nearest-neighbour approach. This is run during Algorithm \ref{alg5}.

\begin{algorithm}[H]
\caption{Predict class of feature vector among seen classes}
\label{alg4}
\hspace*{\algorithmicindent} \textbf{Input:} $F$ a feature vector, $\mathcal{W}_s^{(k)}$, $\mathrm{RF}$ an optional random forest, $\mathrm{LR}$ an optional logistic regression.\\
\hspace*{\algorithmicindent} \textbf{Output:} Class prediction for feature vector
\begin{algorithmic}[1]
    \If{$\mathrm{RF}$ available}
        \State return $\mathrm{RF}(F)$
    \ElsIf{$\mathrm{LR}$ available}
        \State return $\mathrm{LR}(F)$
    \Else
        \State return nearest distance class from $\mathcal{W}_s^{(k)}$
    \EndIf
\end{algorithmic}
\end{algorithm}

\section{Further Theoretical Details}
\label{app:sec:theory}
\subsection{Why Use \texorpdfstring{$E_m$}{Em} over a Consensus Voting?}
To remove the bias toward seen classes, the simplest ensembling method that comes to mind is to declare the class of a new example as `seen` if all the classifiers $\ensclass{k}$ call it `seen'. Let us call this consensus voting rule $E_{\text{all}}(x)$.

Then, assuming independence of the predictors $\ensclass{k}$, it follows that
\begin{align*}
\mathbb{P}(E_{\text{all}}(x)=\text{`unseen'} \mid w\in\mathcal{W}_u)&=1-\prod_{k \in [L]} \beta_k \ge 1 - \beta^L,
\end{align*}
provided a mild upper bound $\beta_k \le \beta < 1$ holds on the Type II errors of the predictors.

Thus the consensus rule $E_{\text{all}}$ is able to remove the bias toward `seen' classes. However, it is too stringent when it comes to correctly identifying `seen' classes. Recall that $\mathcal{L}(w):=\{k \mid w\in\mathcal{W}_s^{(k)}\}$ and set $\ell(w) := \#\mathcal{L}(w)$. In other words, $\ell(w)$ many predictors had $w$ as a seen class in their training sets.
% Then, for $w\in W_s$, we say that $\ell(w):=|\mathcal{L}(w)|$ is the number of $\mathcal{W}_s^{(k)}$ containing $w$.
Then,
\begin{align*}
\mathbb{P}(E_{\text{all}}(x)=\text{`seen'}|w\in\mathcal{W}_s)&=\mathbb{P}(\text{all }\ensclass{k}\text{ call }x\text{ seen})\\
&=\prod_{k:w\in\mathcal{W}_s^{(k)}}(1-\alpha_k)\prod_{k:w\not\in\mathcal{W}_s^{(k)}}\beta_k\\
&\geq(1-\alpha)^{\ell(w)}\beta^{L-\ell(w)},
\end{align*}
where the last inequality holds if $\alpha_k \le \alpha$ and $\beta_k \ge \beta$ for each $k \in [L]$. The takeaway is that the accuracy in identifying `seen' classes can be quite far from $1$.

The above computation suggests the use of a less stringent voting rule such as $E_m$ with an appropriate choice of $m$.
% so we need to use some threshold instead of outputting `unseen' if and only if some $\ensclass{k}$ outputs unseen.

\subsection{Proofs of Various Results}
Recall that
\[
    \Delta(w) := \sum_{k \in \mathcal{L}(w)} \gamma_k, \qquad
    \Delta := \min_{w \in \mathcal{W}_s} \Delta(w).
\]
% Notice that $\Delta$ is a lower bound for the extra expected number of `seen' outputs when $w\in\mathcal{W}_s$. We make use of a concentration inequality to prove a lower bound on accuracy.
\begin{proof}[Proof of Theorem~\ref{thm:arbit}]
Let
\[
    X:= \sum_{k \in [L]} \mathds{1}(\ensclass{k}(x) = \text{`seen'}).
\]
Then $X$ is a sum of independent Bernoulli random variables. Note that
\begin{align*}
\mu_s:=\mathbb{E}[X \mid w\in\mathcal{W}_s] &= \sum_{k\in\mathcal{L}(w)}(1-\alpha_k)+\sum_{k\in[L]\setminus\mathcal{L}(w)}\beta_k\\
&= \sum_{k \in \mathcal{L}(w)} (1 - \alpha_k - \beta_k) + \sum_{k \in [L]} \beta_k \\
&= \Delta(w) + \mu \\
&\ge \Delta+\mu.
\end{align*}
On the other hand,
\begin{align*}
\mu_u:=\mathbb{E}[X \mid w\in\mathcal{W}_u]&=\sum_{k\in[L]}\beta_k=\mu
\end{align*}
For our choice of $\delta$, it may be seen that
\[
    (1+\delta)\mu=\frac{2\mu(\mu+\Delta)}{2\mu+\Delta}\leq\frac{2\mu\mu_s}{2\mu+\Delta}=(1-\delta)\mu_s. 
\]
Applying the Chernoff bounds, we have
\begin{align*}
\mathbb{P}(E_m(x) = \text{`unseen'} \mid w \in \mathcal{W}_s) &= \mathbb{P}(X < m \mid w\in\mathcal{W}_s) \\
&= \mathbb{P}(X < (1+\delta)\mu \mid w\in\mathcal{W}_s) \\
&\le\mathbb{P}[X<(1-\delta)\mu_s|w\in\mathcal{W}_s]\\
&\leq e^{-\delta^2\mu_s/2}.
\end{align*}
Again by the Chernoff bounds,
\begin{align*}
\mathbb{P}(E_m(x) = \text{`seen'} \mid w \not\in \mathcal{W}_s) &= \mathbb{P}(X \ge m \mid w\in\mathcal{W}_u) \\
&=\mathbb{P}(X\geq(1 + \delta) \mu \mid w\in\mathcal{W}_u) \\
&\le e^{-\delta^2\mu/(2+\delta)}.
\end{align*}
% \color{red}
% \begin{align*}
% &\frac{1 - \delta}{1 + \delta} \ge \frac{\mu}{\Delta + \mu} \\
% \iff & (1 - \delta) (\Delta + \mu) \ge (1 + \delta) \mu \\
% \iff & \Delta \ge (\Delta + 2\mu) \delta \\
% \iff & \delta \le \frac{\Delta}{\Delta + 2\mu}.
% \end{align*}
% \color{black}
This completes the proof.
\end{proof}

\begin{proof}[Proof of Theorem~\ref{thm:random}]
We note that under the random allocation model, the indicators $X_{w, k} := \mathds{1}(w \in \mathcal{W}^{(k)}_s)$ are i.i.d. $\mathrm{Ber}(p)$ random variables. Now
\[
    \Delta(w) = \sum_{k \in [L]} \gamma_k X_{w, k},
\]
where $\gamma_k = 1 - \alpha_k - \beta_k$. Note that the $\gamma_k$'s a priori depend on the allocated classes $\mathcal{A} := (\mathcal{W}^{(k)}_s)_{k \in [L]}$. Given an allocation $\mathcal{A}$, we may repeat the argument used in the proof of Theorem~\ref{thm:arbit}. Let $\Delta_0$ denote a deterministic threshold such that $\mathbb{P}(\Delta \le \Delta_0)$ is small. Then, with $\delta = \frac{\Delta_0}{2 \mu + \Delta_0}$ and $m = \lceil \mu (1 + \delta) \rceil$, we have
\begin{align*}
    \mathbb{P}(E_m(x) = \text{`unseen'} \mid w \in \mathcal{W}_s) &= \mathbb{E}\big[\mathbb{P}(E_m(x) = \text{`unseen'} \mid w \in \mathcal{W}_s, \mathcal{A})\big] \\
    &= \mathbb{E}\big[\mathbb{P}(E_m(x) = \text{`unseen'} \mid w \in \mathcal{W}_s, \mathcal{A}) \mathds{1}(\Delta < \Delta_0)\big] \\
    & \qquad +  \mathbb{E}\big[\mathbb{P}(E_m(x) = \text{`unseen'} \mid w \in \mathcal{W}_s, \mathcal{A}) \mathds{1}(\Delta \ge \Delta_0)\big] \\
    &\le \mathbb{P}(\Delta < \Delta_0) + \mathbb{E}\big[e^{-\delta^2 (\mu + \Delta_0) / 2}\big].
\end{align*}
Similarly,
\begin{align*}
    \mathbb{P}(E_m(x) = \text{`unseen'} \mid w \in \mathcal{W}_s) &\le \mathbb{P}(\Delta < \Delta_0) + \mathbb{E}\big[e^{-\delta^2 \mu / (2 + \delta)} \big].
\end{align*}

If we further assume that the $\alpha_k$'s and the $\beta_k$'s do not depend on the class allocations $\mathcal{A}$, then the $\Delta(w)$'s are independent random variables, and hence by Hoeffding's inequality, for any $\eta \in (0, 1)$,
\begin{align*}
    \mathbb{P}\bigg(\Delta(w) \le (1 - \eta) p \sum_{k \in [L]} \gamma_k\bigg) &= \mathbb{P}(\Delta(w) \le (1 - \eta) \mathbb{E}[\Delta(w)]) \\
    &\le \exp\bigg(-\frac{2 \eta^2 p^2 (\sum_{k \in [L]} \gamma_k)^2}{\sum_{k \in [L]} \gamma_k^2}\bigg).
\end{align*}
It follows by a union bound that
\[
    \mathbb{P}\bigg(\Delta \le (1 - \eta) p \sum_{k \in [L]} \gamma_k\bigg) \le |\mathcal{W}_s|\exp\bigg(-\frac{2 \eta^2 p^2 (\sum_{k \in [L]} \gamma_k)^2}{\sum_{k \in [L]} \gamma_k^2}\bigg).
\]
Thus we may set $\Delta_0 := (1 - \eta) p \sum_{k \in [L]} \gamma_k$ and obtain
\small{\begin{align*}
\mathbb{P}(E_m(x) = \text{`unseen'} \mid w \in \mathcal{W}_s) \le |\mathcal{W}_s|\exp\bigg(-\frac{2 \eta^2 p^2 (\sum_{k \in [L]} \gamma_k)^2}{\sum_{k \in [L]} \gamma_k^2}\bigg) + e^{-\delta^2 (\mu + \Delta_0) / 2}.
\end{align*}}
Similarly,
\small{\begin{align*}
\mathbb{P}(E_m(x) = \text{`seen'} \mid w \in \mathcal{W}_s) \le |\mathcal{W}_s|\exp\bigg(-\frac{2 \eta^2 p^2 (\sum_{k \in [L]} \gamma_k)^2}{\sum_{k \in [L]} \gamma_k^2}\bigg) + e^{-\delta^2\mu/(2+\delta)}.
\end{align*}}
This completes the proof.
\end{proof}

% \color{red} $\Delta$ is random below, so the choice of $m$ has to be based on the lower bound. This can be removed.
% \begin{theorem}[Random allocation]\label{thm:random2}
% Suppose $\Delta=\min_{w\in\mathcal{W}_s}\ell(w)-\left(\sum_{k\in\mathcal{L}(w)}\alpha_k+\beta_k\right)\geq\varepsilon\min_w\ell(w)$.

% Then, for each $\delta\in[0,1]$, with probability at least $1-|\mathcal{W}_s|e^{-\delta^2 p L/2}$, we have $\Delta\geq\varepsilon(1-\delta)p L$.

% Setting $\mu=\sum_{k\in[L]}\beta_k$ and $\delta'=\frac{\Delta}{2\mu+\Delta}$ and $m=\mu(1+\delta')$, the ensemble classifier $E_m$ has accuracy $\geq1-e^{-\delta'^2\mu/(2+\delta')}$.
% \end{theorem}

% \begin{proof}
% Applying Chernoff bounds, we note for each $w$ that
% \[\mathbb{P}(\ell(w)\leq(1-\delta)p L)\leq e^{-\delta^2 p L/2}.\]
% Hence, by union bound,
% \[\mathbb{P}(\min_w\ell(w)\leq(1-\delta) p L)\leq|\mathcal{W}_s|e^{-\delta^2 p L/2}.\]
% The remaining part follows from Lemma \ref{ensclass-acc}.
% \end{proof}
% \color{black}

\subsection{A Result under Exchangeability of the Predictors \texorpdfstring{$\psi_k$}{psi\_k}}

% $\alpha_k,\beta_k$ can be different if we use different statistical methods to train different classifiers.

% We see $\Delta$ is a weighted sum of $L$ Bernoulli random variables of mean $p$, with deterministic weights $1-\alpha_i-\beta_i\geq\varepsilon>0$.

One issue with the previous analysis is that we have been assuming independence of the predictors $\psi_k$ (for instance, by sample splitting, which comes at the cost of some statistical efficiency). In practice, the different predictors in the ensemble may be assigned common training samples, rendering them not independent of each other. It turns out that we can prove similar results under the weaker assumption of exchangeability.

%\textcolor{red}{TODO add exchangable lemma cite paper theorem 4 or maybe 5 (alpha and beta have to be the same?) do it in a different one?} % https://arxiv.org/pdf/2404.06457

\begin{theorem}
Fix $x,w(x)$. Suppose $(\alpha_1,\beta_1),\dots(\alpha_L,\beta_L)$ are exchangeable random variables with common mean $(\alpha,\beta)$.

Then, for $m=L\left(\beta+\frac{p}{2}(1-\alpha-\beta)\right)$, the ensemble classifier $E_m$ is such that
\begin{align*}
    \mathbb{P}(E_m(x) = \text{`unseen'} \mid w \in \mathcal{W}_s) &\le \exp\left(-\frac{p^2L(1-\alpha-\beta)^2(L-H_L)}{2(L-1)}\right)\\
    \mathbb{P}(E_m(x) = \text{`seen'} \mid w \in \mathcal{W}_u) &\le \exp\left(-\frac{p^2L(1-\alpha-\beta)^2(L-H_L)}{2(L-1)}\right).
\end{align*}

%=d/(2u+d)-2u delta-d delta
%If $0<m<\min_w\ell(w)$ then the accuracy of the classifier is at least $1-\exp(C_1m)-\exp(C_2(\min_w\ell(w)-m))$.
\end{theorem}

Here, $(\alpha_1,\beta_1),\dots(\alpha_L,\beta_L)$ being exchangeable means that for any permutation $\sigma$ of $[L]$, we have
\[(\alpha_1,\dots,\alpha_L,\beta_1,\dots,\beta_L) \stackrel{d}{=} (\alpha_{\sigma(1)},\dots,\alpha_{\sigma(L)},\beta_{\sigma(1)},\dots,\beta_{\sigma(L)}).\]

\begin{proof}
As before, let
\[
    X = \sum_{k \in [L]} \mathds{1}(\ensclass{k}(x) = \text{`seen'}).
\]
% hmm is it even clear
We analyze this variable in the cases $w\in\mathcal{W}_s$ and $w\in\mathcal{W}_u$.

From \cite[Theorem 4]{Barber2024HoeffdingAB}, let $v\in\mathbb{R}^L$ and $X_1,\dots,X_L\in[-1,1]$ be exchangeable, where $L\geq2$. Then,

\[\mathbb{P}\left\{\sum_{i=1}^Lv_i(X_i-\mathbb{E}[X_i])\geq\|v\|_2\sqrt{2\left(\frac{L-1}{L-H_L}\right)\log(1/\delta)}\right\}\leq\delta\]
where $H_L=\sum_{i=1}^L\frac{1}{i}\sim O(\log L)$.

We apply this in the case where $X_i$ is a Rademacher random variable as follows:
\[X_i=\begin{cases}
1&\psi_k\text{ calls }x\text{ seen}\\
-1&\text{otherwise}
\end{cases}\]
so we get:
\begin{align*}
\mathbb{E}[X_i]&=\begin{cases}2p(1-\alpha)+2(1-p)\beta-1&w\in\mathcal{W}_s\\
2\beta-1&w\in\mathcal{W}_u
\end{cases}\\
X-\mathbb{E}[X]&=\frac{1}{2}\sum_{i=1}^LX_i-\mathbb{E}[X_i].
\end{align*}
Setting $v=(1/2,\dots,1/2)$, we obtain
\[\mathbb{P}\left\{X-\mathbb{E}[X]\geq\frac{1}{2}\sqrt{2L\left(\frac{L-1}{L-H_L}\right)\log(1/\delta)}\right\}\leq\delta.\]
Setting $v=(-1/2,\dots,-1/2)$, we obtain
\[\mathbb{P}\left\{X-\mathbb{E}[X]\leq-\frac{1}{2}\sqrt{2L\left(\frac{L-1}{L-H_L}\right)\log(1/\delta)}\right\}\leq\delta.\]

Hence, choosing our threshold $m=L\left(\beta+\frac{p}{2}(1-\alpha-\beta)\right)$, we can get
\[
    \frac{1}{2}\sqrt{2L\left(\frac{L-1}{L-H_L}\right)\log(1/\delta)} = \frac{p L}{2}(1-\alpha-\beta),
\]
i.e.
\[
    \delta = \exp\left(-\frac{p^2L(1-\alpha-\beta)^2(L-H_L)}{2(L-1)}\right)
\]
which is our false positive rate/false negative rate. We notice $\delta\sim\exp(-cL)$ for some constant $c$.
\end{proof}

We can see that the prediction errors shrink exponentially in $L$, making the ensemble predictor a viable approach.

% Each of the classifiers $\ensclass{k}$ is wrong about $w(x)\in\mathcal{W}_s$ some small amount of the time. Ensembling methods allow us to combine the correct predictions while having robustness against the wrong outputs.

\end{document}